%% file: paper.tex
\documentclass[12pt]{article}

\usepackage{amsfonts}
\usepackage{fullpage}
\usepackage{graphicx}
\usepackage[ruled]{algorithm2e}
\usepackage{multirow}
\usepackage{graphicx}
\usepackage{mathptmx}      
\usepackage{amsmath,amsthm,amssymb}
\usepackage{xspace}
\usepackage{mathtools}
\usepackage{tikz}
\usepackage{subfig}
\usepackage{acronym}

\def\seq{\mathaccent"017E }
\newcommand{\qes}[1]{\reflectbox{\ensuremath{\seq{\reflectbox{\ensuremath{#1}}}}}}
\newcommand{\lexeq}{\ensuremath{\preceq_{\it lex}}\xspace} 
\newcommand{\lex}{\ensuremath{\prec_{\it lex}}\xspace}
\newcommand{\mmfeq}{\ensuremath{\preceq}\xspace}

\newcommand{\mmf}{\ensuremath{\prec}\xspace}
\newcommand{\mmfmax}{\ensuremath{\prec_{\textit{max}}}\xspace}
\newcommand{\rank}{\ensuremath{\operatorname{rank}}\xspace}
\newcommand{\MMFSA}{\textsc{MaxMinFair\_SA}\xspace}
\newcommand{\MMFSAGIP}{\textsc{MMF\_SA\_GLBOP}\xspace}
\newcommand{\MMFSANOGIP}{\textsc{MMF\_SA\_LSAP}\xspace}
\newcommand{\TEMP}{{\ensuremath \vartheta}}

\newcommand{\solutions}{{\ensuremath \mathcal{S}}\xspace}
\newcommand{\solution}{{\ensuremath S}\xspace}
\newcommand{\weight}{{\ensuremath W}\xspace}
\newcommand{\sumweight}{{\ensuremath W'}\xspace}
\newcommand{\instance}{{\ensuremath I}\xspace}
\newcommand{\sortedsequences}{{\ensuremath \operatorname{Seq}}\xspace}

\newcommand{\COST}[1]{\ensuremath{c_{\text{S}{#1}}}\xspace}

\usepackage{xspace}

\acrodef{GLBOP}{generalized lexicographic bottleneck optimization problem}
\acrodef{LBOP}{lexicographic bottleneck optimization problem}
\acrodef{LSAP}{linear sum assignment problem}
\acrodef{CB-CTT}{curriculum-based course timetabling}
\acrodef{MMF-CB-CTT}{max-min fair curriculum-based course timetabling}
\acrodef{SOP}{min-sum optimization problem}
\acrodef{LBAP}{lexicographic bottleneck assignment problem}
\acrodef{LVOP}{lexicographic vector optimization problem}

\usepackage{rotating}

\makeatletter
\newdimen\@@@tmpa
\newdimen\@@@tmpb
\newdimen\@@@tmpc
\newdimen\@@@tmpd
\def\clap#1#2#3{
       \setbox0=\hbox{#1}\setbox1=\hbox{#2}%
       \@@@tmpa\wd0\@@@tmpb\wd1\advance\@@@tmpa-\@@@tmpb%
       \ifdim\@@@tmpa>0pt\@@@tmpb\wd0%
       \else\@@@tmpb\wd1\fi\@@@tmpc\ht0\@@@tmpd\ht1%
       \advance\@@@tmpc\dp0\advance\@@@tmpd\dp1%
       \advance\@@@tmpc-\@@@tmpd\divide\@@@tmpc2%
       \ifdim\@@@tmpc>0pt\@@@tmpd\ht0\advance\@@@tmpd\dp0\@@@tmpc-\dp0%
       \else\@@@tmpd\ht1\advance\@@@tmpd\dp1\advance\@@@tmpc-\dp0\fi%
       \ifx#3\empty\else\advance\@@@tmpc#3\fi%
       \leavevmode\raise\@@@tmpc\hbox to \@@@tmpb{\rlap{\hbox to \@@@tmpb{\hss%
            \vbox to \@@@tmpd{\vss\box0\vss}\hss}}%
            \hss\vbox to \@@@tmpd{\vss\box1\vss}\hss}%
       }

\makeatother

\newcommand{\Discup}{\mathop{\clap{$\displaystyle\bigcup$}{$\cdot$}{}}}

\newtheorem{definition}{Definition}
\newtheorem{theorem}{Theorem}
\newtheorem{lemma}{Lemma}
\newtheorem{corollary}{Corollary}
\newtheorem{remark}{Remark}

\begin{document}

\title{A Decomposition of the Max-min Fair Curriculum-based Course
Timetabling Problem\thanks{Research funded in parts by the School of
Engineering of the University of Erlangen-Nuremberg.}\\ \Large The Impact of Solving Subproblems to Optimality}

\author{Moritz M\"uhlenthaler\qquad Rolf Wanka \\[2mm]
              Department of Computer Science \\
              University of Erlangen-Nuremberg,
              Germany \\
              \texttt{\{moritz.muehlenthaler,rolf.wanka\}@cs.fau.de}
}

\date{ }

\maketitle

\input{abstract}

\bigskip

\input{01_introduction}
\input{02_background}
\input{03a_decomposition}
\input{04_quantify}
\input{05_evaluation}
\input{06_conclusion}

\bibliographystyle{plain}
\bibliography{paper}   

\end{document}

%% file: abstract.tex
\begin{abstract}
We propose a decomposition of the \ac{MMF-CB-CTT} problem.  The decomposition
models the room assignment subproblem as a generalized \acf{LBOP}. We show that
the generalized \ac{LBOP} can be solved efficiently if the corresponding sum
optimization problem can be solved efficiently. As a consequence, the room
assignment subproblem of the \ac{MMF-CB-CTT} problem can be solved efficiently.
We use this insight to improve a previously proposed heuristic algorithm for
the \ac{MMF-CB-CTT} problem. Our experimental results indicate that using the
new decomposition improves the performance of the algorithm on most of the 21
ITC2007 test instances with respect to the quality of the best solution found.
Furthermore, we introduce a measure of the quality of a solution to a max-min
fair optimization problem. This measure helps to overcome some limitations
imposed by the qualitative nature of max-min fairness and aids the statistical
evaluation of the performance of randomized algorithms for such problems. We
use this measure to show that using the new decomposition the algorithm 
outperforms the original one on most instances with respect to the average
solution quality.
\end{abstract}


%% file: 01_introduction.tex
\section{Introduction}

%
We consider a decomposition approach to a variant of the \acf{CB-CTT} problem.
The \ac{CB-CTT} problem has been proposed in the context of the timetabling
competition ITC2007~\cite{ITC2007} and has since then received a great deal of
attention in the timetabling community. The \ac{CB-CTT} problem models the task
of assigning timeslots and rooms to courses in
the setting of a university, and includes a number of requirements that
typically arise in real-world course timetabling applications. For instance,
courses occurring in the same curriculum must not be taught at the same time
and courses should be assigned to rooms with a sufficient number of seats.  The
combinatorial structure of the problem is quite complex and has been
investigated, for example, using polyhedral studies~\cite{Marecek:11,LL:12}.
Due to the structural complexity, solving instances with a large number of
courses and curricula is typically out of reach for exact methods. However, a
wealth of (meta-)heuristic methods have been applied successfully to generate
high quality timetables, even for large CB-CTT instances, see for
example~\cite{BGS:12,ATS:2008,Muller:2009}. 
In this work, our focus is 
the \acf{MMF-CB-CTT} introduced in~\cite{MW:12}, which replaces the original
optimization objective by a lexicographic bottleneck objective. The goal of the
MMF-CB-CTT problem formulation is to favor \emph{fair} timetables and thus
improve the overall stakeholder satisfaction. The underlying fairness concept
is (lexicographic) max-min fairness. 

It is a common technique to decompose the \ac{CB-CTT} problem into subproblems which
are easier to handle individually~\cite{LL:12,ATS:2008}. The usual approach is
to perform room and time\-slot assignments separately, but other approaches have
been explored as well~\cite{BMPR:10}. The \ac{CB-CTT} problem can be decomposed
into a bounded list coloring problem that models the timeslot assignment and,
for each timeslot, a \acf{LSAP} for assigning the courses in this timeslot to
rooms~\cite{Marecek:11,LL:12}. Unfortunately, there are dependencies between
\ac{LSAP}s for different
timeslots, so an optimal room assignment can only be obtained for a single
timeslot, while the rest of the timetable remains fixed. We show that for an
analogous decomposition of the \ac{MMF-CB-CTT} problem, the room assignment
subproblem for a single timeslot can also be solved efficiently by modeling it
as a \acf{GLBOP}, which is a generalization of the \acf{LBOP}
from~\cite{BR:91}. We show that the \ac{GLBOP} can be solved efficiently if the
corresponding sum optimization problem can be solved efficiently. 

Furthermore, we propose a new measure for the quality of a solution to an
optimization problem with a max-min fairness (lexicographic bottleneck)
objective such as the \ac{MMF-CB-CTT} problem. This measure helps to overcome
some limitations imposed by the qualitative nature of max-min fairness. We use
this measure to determine the average solution quality of a randomized
algorithm for the \ac{MMF-CB-CTT} problem.

We evaluate the quality of the timetables produced by the algorithm \MMFSA
from~\cite{MW:12}, with and without the new decomposition. In the original
algorithm the room assignment subproblem was modeled as an \ac{LSAP}. However,
an optimal solution to the \ac{LSAP} is not necessarily optimal for
corresponding the room assignment subproblem of the \ac{MMF-CB-CTT} problem.
Our experiments indicate that making use of the new decomposition improves the
best produced by the algorithm on 18 out of 21 \ac{CB-CTT} instances from the
ITC2007 competition. We use the aforementioned measure to show that the new
decomposition yields in an improved average solution quality for 16 out of 21
instances. According to the Wilcoxon rank-sum test (one-sided, significance
level $0.01$) \MMFSA using the new decomposition is significantly better than
the original approach on 12 instances.

The remainder of this work is organized as follows: In
Section~\ref{sec:background}, we provide relevant background on the \ac{CB-CTT}
and \ac{MMF-CB-CTT} problem formulations, as well as max-min fairness and the
assignment problem. In Section~\ref{sec:decomp}, we introduce the \ac{GLBOP}
and the decomposition of the \ac{MMF-CB-CTT} problem.  We propose the measure
of the max-min fairness of an allocation in Section~\ref{sec:iso}.
Section~\ref{sec:evaluation} presents our evaluation of the performance impact
of the decomposition for \ac{MMF-CB-CTT} problems. Section~\ref{sec:conclusion}
concludes the paper.


%% file: 02_background.tex
\section{Background}
\label{sec:background}

In this section we will provide some relevant background on fair resource
allocation, university course timetabling and the assignment problem. 

\subsection {Fair Resource Allocation}
\label{sec:fairness}

Fairness comes into play when scarce resources are distributed over a set of
stakeholders with demands. 
The topic of fairness and how to measure it has
received great attention for example in economics, where the
distribution of wealth and income is of interest~\cite{WelfareEconomics}. In
computer science, fairness aspects have been studied for example in the design
of network communication protocols, in particular in the context of bandwidth
allocation and traffic shaping~\cite{Lan:10,Jain:84}. Fairness aspects have
been addressed explicitly for example for various kinds of scheduling problems,
including personnel scheduling~\cite{Smet:12}, sports
scheduling~\cite{Ribeiro:06}, course scheduling~\cite{MW:12} and aircraft
scheduling~\cite{SK:08}. In the context of resource allocation allocation in
general, fairness has been studied in~\cite{BFT:11,Ogryczak:10}. 
Approximation algorithms for fair optimization problems have been studied
in~\cite{Kleinberg:01,KK:00}.

In the next sections of this work we will deal with a fair variant of a
university course timetabling problem from~\cite{MW:12} that builds on the
notion of \emph{max-min fairness}.  Consider the problem of allocating
resources to $n$ stakeholders. A resource
allocation induces an allocation vector $x = (x_1,\ldots,x_n)$, where $x_i$, $1
\leq i \leq n$, corresponds to the amount of resources allocated to stakeholder
$i$. Later on, we will deal with minimization problems almost exclusively, so
unless stated otherwise, we assume that we allocate to the stakeholders a cost,
penalty, or some similar quantity that is to be minimized. In this contexct, an
allocation is \emph{max-min fair}, if the cost that the worst-off stakeholder
has to cover is minimal, and under this condition, the second worst-off
stakeholder covers the minimal cost, and so forth. This concept can be
formalized as follows: Let $\qes x$ denote the sequence containing the entries
of the allocation $x$ sorted in non-increasing order. For allocation two
vectors $x$ and $y$, $x$ is at least as fair as $y$, denoted by $x \mmfeq y$,
if $\qes x \lexeq \qes y$, where $\lexeq$ is the usual lexicographic
comparison.  Now, an allocation $x$ is max-min fair, if $x \mmfeq y$ for all
feasible allocations $y$.

Max-min fairness enforces an efficient resource usage to some extent, since an
improved resource utilization is accepted to the benefit of a stakeholder as
long as it is not at the expense of another stakeholder who is worse-off.
Hence, a max-min fair allocation is Pareto-optimal. One limitation of max-min
fairness is that the concept is purely qualitative, i.\,e., given two
allocations $x$ and $y$, max-min fairness just determines which of the two is
fairer, but not by how much. In order to aid the statistical evaluation of the
performance of algorithms for max-min fair resource allocation problems, in
Section~\ref{sec:iso} we will introduce a metric for the difference in quality
between two allocation vectors which is compatible with the $\mmfeq$-relation.

\subsection{Curriculum-based Course Timetabling}
\label{sec:cb-ctt}
The academic course timetabling problem captures the task of assigning a set of
courses to rooms and timeslots in the setting of a university. In
Section~\ref{sec:decomp} we will focus on decompositions of two
particular variants of the academic course timetabling problem: the
\acf{CB-CTT} problem from track three of the
second international timetabling competition~\cite{ITC2007:CB-CTT}, and its max-min
fair version, \ac{MMF-CB-CTT}, proposed in~\cite{MW:12}. The \ac{CB-CTT}
formulation has attracted a great deal of interest in the research community
and the competition instances are popular benchmarking instances for comparing
different algorithms~\cite{ATS:2008,LL:12,MR:12}. The \ac{MMF-CB-CTT} problem
differs from the basic \ac{CB-CTT} formulation only with respect to the
objective function. We will now introduce some terminology and state
definitions relevant to the later sections of this work.

A \ac{CB-CTT} instance consists of a set of courses, a set of curricula, a set of
rooms, a set of teachers and a set of days. Each day is divided into a fixed
number of timeslots; a day together with a timeslot is referred to as a
\emph{period}. A period together with a room is called a \emph{resource}.
Each course consists of a set of events that need to be scheduled. A course is taught by
a teacher and has a fixed number of students attending it. A course can only be
taught in certain available periods. Each curriculum is a set of courses, no two
of which may be taught in the same period. Each room has a \emph{capacity}, a
maximum number of students it can accommodate. A solution to a \ac{CB-CTT} instance
is a \emph{timetable}, i.\,e., an assignment of the courses to the resources subject to
a number of hard and soft constraints. A timetable that satisfies all hard
constraints is \emph{feasible}. 

Later on, we will deal exclusively with feasible timetables, so we will not
cover the evaluation of hard constraints at all (please refer
to~\cite{ITC2007:CB-CTT} a detailed description). However, some understanding
of the soft constraint evaluation will be helpful later on, so we will touch on
this very briefly. The \ac{CB-CTT} problem formulation features the following soft
constraints:
\begin {enumerate}
\item[S1\label{itm:S1}] \emph{RoomCapacity}:
Each lecture should be assigned to a room of sufficient size.
\item[S2\label{itm:S2}] \emph{MinWorkingDays}:
The individual lectures of each course should be distributed over a certain
minimum number of days.
\item[S3\label{itm:S3}] \emph{IsolatedLectures}:
For each curriculum, all courses in the curriculum should be scheduled in adjacent periods.
\item[S4\label{itm:S4}] \emph{RoomStability}:
The lectures of each course should be held in the same room.
\end {enumerate}
The violation of a soft constraint results in a ``penalty'' for the timetable.
The total penalty of a timetable $\tau$ is aggregated by the objective function
$c$ which just sums the penalties for the individual soft constraint violations:
\begin{equation}
\label{eq:cbctt_objective}
	c(\tau) = \sum_{1 \leq i \leq 4} \COST{i}(\tau),
\end{equation}
where $\COST{1},\ldots,\COST{4}$ are the penalties determined by the soft
constraints (S1)--(S4). 
The relative importance of the different soft constraints is set by a weight
factor for each soft constraint. Since the weights will be of no relevance to
our arguments later on, we assume that appropriate weighting has been applied within
$\COST{1},\ldots,\allowbreak \COST{4}$.
A detailed specification of the penalty functions 
can be found in~\cite{ITC2007:CB-CTT}. 

\begin{definition}[\ac{CB-CTT} Problem]
Given a \ac{CB-CTT} instance $\instance$, find a feasible timetable $\tau$ such that $c(\tau)$
is minimal.
\end{definition}

A max-min fair variant of the \ac{CB-CTT} problem was defined in~\cite{MW:12}. Given
a \ac{CB-CTT} instance with curricula $u_1,\ldots,u_k$. The allocation vector of a
timetable $\tau$ is given by: 
\begin{equation} 
	\label{eq:allocation} A(\tau) = (c(u_1,\tau), c(u_2,\tau), \ldots, c(u_k,\tau)) \enspace,
\end{equation}
where $c(u_j,\tau) = \sum_{1 \leq i \leq 4} c_{\text{S}i}(u_j,\tau)$, $i \in \{1,\ldots,4\}$, is the
\ac{CB-CTT} objective function restricted to the events of the courses in curriculum
$u_j$, $j \in \{1,\ldots,k\}$.

\begin{definition}[\ac{MMF-CB-CTT} Problem]
Given a \ac{CB-CTT} instance $\instance$, find a feasible timetable $\tau$ such that $A(\tau)$
is max-min fair.
\end{definition}

\subsection{The Assignment Problem}
\label{sec:ap}

The assignment problem is a classical problem in combinatorial optimization
which appears in many applications, for example personnel scheduling, job 
scheduling and object tracking, just to name a few. For a comprehensive
overview of the body of research on the assignment problem and the applications
see~\cite{BAS:09,Pentico:07}. In \ac{CB-CTT} problem for example, the
assignment problem appears as a subproblem~\cite{LL:12,Lu:10}.  There exist
polynomial-time algorithms for many variants of the assignment problem. 

Let $A = B = \{1,\ldots,n\}$, for some $n \in \mathbb{N}$. An assignment of the elements
of $A$ to the elements of $B$ is a bijection $\sigma: A \rightarrow B$. 
Typically, assignment problems are optimization problems, i.\,e., among all
bijections from $A$ to $B$, we are looking for one that is optimal with respect
to a certain objective function. In the context of (fair) curriculum-based
course timetabling, we are in particular interested in two variants of the
assignment problem, namely the \acf{LSAP} and the \acf{LBAP}.

\begin{definition}[\ac{LSAP}]
Given a cost function $c: A\times B \rightarrow \mathbb{N}$, find a bijection
$\sigma: A \rightarrow B$ such that $\sum_{i = 1}^n c(i, \sigma(i))$ is
minimal.
\end{definition}

There exist various algorithms for solving LSAPs efficiently, including the
well-known Hungarian algorithm~\cite[p.~248ff]{Papadimitriou:98} and network
flow algorithms~\cite{GT:89}. In the following, let $T_{\rm LSAP}(n)$ be the
time complexity of solving an \ac{LSAP} instance with $|A| = |B| = n$.

When solving \ac{MMF-CB-CTT} problems using the decomposition proposed in the next
section, the task of finding max-min fair assignments occurs as a
subproblem. An assignment $\sigma: A \rightarrow B$ is called max-min fair,
if for any assignment $\sigma': A \rightarrow B$ we have $\vec c(\sigma)
\mmfeq \vec c(\sigma')$, where $\vec c(\sigma) = (c(i, \sigma(i)))_{i=1,\ldots,n}$
and $\mmfeq$ is the max-min fair comparison from Section~\ref{sec:fairness}. 

\begin{definition}[\ac{LBAP}]
Given a cost function $c: A\times B \rightarrow \mathbb{R}$, find a max-min fair 
bijection $\sigma: A \rightarrow B$.
\end{definition}

A LBAP can be transformed into a \ac{LSAP} by scaling the cost values appropriately.
This results in an exponential blow-up of the cost values, which may
be undesirable in practical applications~\cite{BR:91}. Alternatively, an LBAP
can be reduced to a lexicographic vector assignment, which
belongs to the class of \textit{algebraic assignment problems}~\cite{BZ:80}.
Using this reduction, a given \ac{LBAP} with a cost function $c$ can be solved in time
$O(kn^3)$, where $k$ is the number of distinct values attained by
$c$~\cite{Croce:99}. The reduction is straightforward: For each $j \in
\{1,\ldots,k\}$, let $e_j \in \mathbb{N}^k$ be the vector whose $j$-th
component is $1$ and all other components are $0$. The cost function $c$ is replaced by a vector-valued function
$c': A\times B \rightarrow \mathbb{N}^k$ such that $c'(a, \sigma(a)) = 
e_j$ if $c(a, \sigma(a))$ is the $j$-th largest value attained by $c$. An
assignment $\sigma$ that yields a lexicographically minimal cost vector
$\sum_{i=1}^n c'(i, \sigma(i))$ is an optimal solution to the corresponding
LBAP.


%% file: 03a_decomposition.tex
\section{Problem Decomposition}
\label{sec:decomp}

It is a common approach to decompose the \ac{CB-CTT} problem in a way that room and
timeslot assignment is preformed separately, see for
example~\cite{LL:12,Lu:10}.  For courses in a single timeslot, an optimal room
assignment can be determined by solving a \ac{LSAP} instance. In this section, we
establish a similar result for the \ac{MMF-CB-CTT} problem: An optimal room
assignment for the courses in a single timeslot can be determined by solving a
(generalized) \ac{LBOP} instance. In the spirit of Benders decomposition
approach~\cite{Benders:62}, we further consider decompositions of sum
optimization problems into master and subproblems such that the subproblems can
be solved efficiently. We derive sufficient conditions under which such a
decomposition of a sum optimization problem carries over to its generalized
lexicographic bottleneck counterpart.

In the following, we consider a generalization of the \ac{LBOP}
in~\cite{BR:91}. Let $E = \{e_1,\ldots,e_m\}$ be the ground set and let
$\solutions = \{\solution_1,\solution_2,\ldots\}$ be the set of
feasible solutions.  We assume that each solution $\solution \in \solutions$
contains exactly $n$ items from the ground set, that is $\solution \subseteq E$
and $|\solution| = n$. Furthermore, let $N \subset \mathbb{N}_0$ be a finite
set of the natural numbers. The weight function $w:\solutions \rightarrow
\mathcal{M}(N)$ assigns to a feasible solution a weight, that is a finite
multiset of the numbers in $N$. The weight of a solution $\solution \in
\solutions$ is the disjoint union of the individual weights,
\[
	\weight(\solution) = \Discup_{e \in \solution} w(e)\enspace.
\]
The weights of two solutions $\solution_1,\solution_2 \in \solutions$ can be
compared by arranging the items in $\weight(\solution_1)$ and
$\weight(\solution_2)$ in non-increasing order and performing a lexicographic
comparison of the sorted sequences. We can essentially the comparison $\mmfeq$
from Section~\ref{sec:fairness}, but the sorted sequences do not necessarily
have the same length. Let $\sortedsequences(N)$ be the set of
finite sequences on the alphabet $N$ that are arranged in non-increasing order.
For two sequences $s, s' \in \sortedsequences(N)$, $s \lexeq s'$ iff one of the
following is true: i) $s = s'$, ii) $s$ is a prefix of $s'$, iii) there is a
decomposition $s = zuv$, $s' = zuw$ such that $z$ is a maximal prefix of $s$
and $s'$, $u, v \in N$ and $u < v$.
Due to the conceptual similarity, we will use the symbol $\mmfeq$ for the
comparison with respect to max-min fairness in the new setting as well. The
\ac{GLBOP} is the following problem: 
\[
	\label{eq:glbop}
	\tag{\ac{GLBOP}}
	\min_{\mmfeq} \;\weight(\solution) \quad\text{ s.t.\ } \solution \in \solutions\enspace.
\]
Note that the minimum weight is determined according to the comparison {\mmfeq}. If
$|w(e)| = 1$ for each $e \in E$ then we have a \ac{LBOP} as defined
in~\cite{BR:91}.

Consider a weight function of the form $\sumweight:\solutions \rightarrow
\mathbb{N}$. Then the \ac{SOP} is the following problem:
\[
	\tag{\ac{SOP}}
	\min_{\leq} \;\sum_{e \in \solution} \sumweight(e) \quad\text{ s.t.\ } \solution \in \solutions\enspace.\label{eq:sop}
\]
Let $T_{n,m}$ be the time required for solving~\eqref{eq:sop}.

\begin{theorem}
\label{thm:glbop}
A \ac{GLBOP} instance can be solved in time $O(|N|\cdot T_{n,m})$.
\end{theorem}
\begin{proof}
Following the vectorial approach of Della Croce et al.\ in~\cite{Croce:99},
reduce the \ac{GLBOP} to a \ac{LVOP}. Let $t = |N|$ and for $1,\leq i \leq t$,
let $t_i$ denote the $i$-th largest item in $N$. Let the function $f:
\mathcal{M}(N) \rightarrow \mathbb{N}^t$ assign to a multiset $v$ of the items 
$N$ a vector $(v_1,\ldots,v_t) \in \mathbb{N}^t$ such that
\[
	v_j = \operatorname{mult}_v(t_j), \quad j=1,\ldots,t\enspace,
\]
where $\operatorname{mult}_v:N \rightarrow \mathbb{N}$, and for each $a \in
\mathbb{N}$, $\operatorname{mult}_v(a)$ is the multiplicity of $a$ in
$v$. Now, we have to solve the following problem:
\[
\label{eq:lvop}
	\tag{LVOP}
	\min_{\lexeq} \;\sum_{e \in \solution} f(\weight(e)) \quad\text{ s.t.\ } \solution \in \solutions\enspace.
\]

We show that a solution $\solution \in \solutions$ is an optimal solution
to~\eqref{eq:lvop} if and only if it is an optimal solution
to~\eqref{eq:glbop}.  Suppose for a contradiction that $\solution \in
\solutions$ is an optimal
solution to~\eqref{eq:lvop} that is not optimal to~\eqref{eq:glbop}. Then there
is a solution $\solution' \in \solutions$ such that $\weight(\solution') \mmf
\weight(\solution)$. As a consequence, $f(\weight(\solution')) \lex
f(\weight(\solution))$ by the construction of the cost vectors above. The
``only if'' part can be shown analogously. This is a contradiction to the
optimality of $\solution'$. The problem~\eqref{eq:lvop} can be solved in time
$O(t\cdot T_{n,m})$ because each elementary operation involving the weights is
now performed on a vector of length $t$.
\end{proof}

As noted by Della Croce et al.\ in~\cite{Croce:99}, the vectorial approach is
essentially cost scaling. The construction of the cost vectors above enables us
to naturally handle multisets as cost values. We will see shortly that in the
context of the \ac{MMF-CB-CTT} problem, if an item of the ground set has a weight of
cardinality $k$, then choosing this item to be part of the solution concerns
$k$ different stakeholders.

Consider the following decomposition of the \ac{MMF-CB-CTT} problem: We
isolate, for a single period, the assignment of courses to rooms, from the rest
of the problem. So, the task is to find an optimal room assignment for a given
period, assuming that the rest of the timetable is fixed.  Optimizing the rest
of the timetable can be considered the master problem corresponding to the room
assignment subproblem for a particular period.  The room assignment subproblem
of the \ac{CB-CTT} problem is a \ac{LSAP} which can be solved efficiently (see
Section~\ref{sec:ap}). Our goal is to show that the room assignment subproblem
of the \ac{MMF-CB-CTT} problem can also be solved efficiently. In the
following, let $\instance$ be a \ac{MMF-CB-CTT} instance, where $C$ is the set
of courses, $R$ is the set of rooms, $P$ is the set of periods, and $U
\subseteq \mathcal{P}(C)$ are the curricula. By $C_p$ we denote the set of
courses scheduled in the period $p \in P$. Please note that $C_p$ is determined
by the solution of the master problem. Furthermore, let $U_e = \{ u \in U \mid
e \in u \}$.  

\begin{theorem}
\label{thm:ras}
The room assignment subproblem of the \ac{MMF-CB-CTT} problem is a \ac{GLBOP}.
\end{theorem}
\begin{proof}
We construct a \ac{GLBOP} that models the room assignment subproblem for a
fixed period $p$. We can assume that $|C_p| = |R|$ since if not, we can add
suitables dummy nodes to either $C_p$ or $R$. Let $G = (C_p \cup R, E)$ be a
complete bipartite graph, where $E = \{ \{e, r\} \mid e \in C, r \in R\}$. 
The ground set of the \ac{GLBOP} is $E$ and the feasible solutions $\solutions$
are all perfect matchings in $G$. Assigning a course $e \in C_p$ to a room $r
\in R$ completes the timetable $\tau$ from the perspective of
all curricula in $U_e$ and therefore determines their cost entries $c(u, \tau)$
for each $u \in U_e$. We denote the cost entry $c(u, \tau)$, which can be
determined after the assignment of $e$ to $r$, by $c_{e \rightarrow r}(u)$.
Thus, the weight $w(e, r)$ of an item $\{e, r\} \in E$ is the
multiset 
\[
	w(e, r) = \Discup_{u \in U_e} \{ c_{e \rightarrow r}(u) \}\enspace.
\]
The costs entries of the
curricula that do not contain any course in $C_p$ are not altered and have no
influence on the optimality of a particular room assignment.
Therefore, the room assignment subproblem of the \ac{MMF-CB-CTT} problem can be
written as:
\[
\label{eq:ras}
\min_{\mmfeq} \quad \Discup_{\{e, r\} \in \solution} w(e, r) \quad
	\text{s.\,t. }  \solution \in \solutions\enspace,
\]
which is a \ac{GLBOP}. 
\end{proof}

Note that according to the \ac{MMF-CB-CTT} problem formulation, a course can
be assigned to any of the rooms. If desired, room availabilities can be added to
the model in a straightforward manner: If a room $r$ is unavailable for a
particular course $e$, then the edge $\{e, r\}$ in the ground set is assigned a
suitably large weight. 

Figure~\ref{fig:rascosts} shows an example of a simple room assignment
subproblem of the \ac{MMF-CB-CTT} problem modelled as a \ac{GLBOP}. There is
only a single period $p$, two courses, $C = C_p =
\{c_1, c_2\}$, and two rooms $R = \{r_1, r_2\}$. However, the course $e_2$ is
in two curricula and thus determines the cost entries of two stakeholders in
the overall allocation vector. The cost on each edge connected to $e_1$ shows
the costs generated for each of the two curricula when assigning $e_2$ to $r_1$
or $r_2$. Figures~\ref{fig:ap1} and~\ref{fig:ap2} are \ac{LBAP} instances that
reflect only costs for one of the two curricula of $e_2$.  The assignments
highlighted are both optimal solutions to the individual \ac{LBAP}s, but none
of them is optimal for the room assignment shown in Figure~\ref{fig:ap3}.

\begin{figure}
\begin{center}
	\subfloat[Room assignment subproblem: optimal cost 7, 5, 4]{
		\begin{tikzpicture}[vertex/.style={node distance=7.5em},weight/.style={auto,midway}]
			\node[vertex] (c1) [] {$r_1$};
			\node[vertex] (c2) [right of=c1] {$r_2$};
			\node[vertex] (r1) [above of=c1] {$e_1$};
			\node[vertex] (r2) [above of=c2] {$e_2$};

			\draw[] (c1) -- (r1) node [weight] {5};
			\draw[very thick,dashed] (c2) -- (r1) node [weight,above,near end,pin=above:{7}] {};
			\draw[very thick,dashed] (c1) -- (r2) node [weight,above,near end,pin=above:{5,4}] {};
			\draw[] (c2) -- (r2) node [weight,right] {7,6};
		\end{tikzpicture}
		\label{fig:ap3}
	}
	\hspace{2em}
	\subfloat[\ac{LBAP} 1: optimal cost 7,5]{
		\begin{tikzpicture}[vertex/.style={node distance=7.5em},weight/.style={auto,midway}]
			\node[vertex] (c1) [] {$r_1$};
			\node[vertex] (c2) [right of=c1] {$r_2$};
			\node[vertex] (r1) [above of=c1] {$e_1$};
			\node[vertex] (r2) [above of=c2] {$e_2$};

			\draw[very thick,dashed] (c1) -- (r1) node [weight] {5};
			\draw[] (c2) -- (r1) node [weight,above,near end,pin=above:{7}] {};
			\draw[] (c1) -- (r2) node [weight,above,near end,pin=above:{5}] {};
			\draw[very thick,dashed] (c2) -- (r2) node [weight,right] {7};
		\end{tikzpicture}
		\label{fig:ap1}
	}
	\hspace{2em}
	\subfloat[\ac{LBAP} 2: optimal cost 6,5]{
		\begin{tikzpicture}[vertex/.style={node distance=7.5em},weight/.style={auto,midway}]
			\node[vertex] (c1) [] {$r_1$};
			\node[vertex] (c2) [right of=c1] {$r_2$};
			\node[vertex] (r1) [above of=c1] {$e_1$};
			\node[vertex] (r2) [above of=c2] {$e_2$};

			\draw[very thick,dashed] (c1) -- (r1) node [weight] {5};
			\draw[] (c2) -- (r1) node [weight,above,near end,pin=above:{7}] {};
			\draw[] (c1) -- (r2) node [weight,above,near end,pin=above:{4}] {};
			\draw[very thick,dashed] (c2) -- (r2) node [weight,right] {6};
		\end{tikzpicture}
		\label{fig:ap2}
	}
\end{center}
\caption[Room Assignment Example]{A room assignment problem example with two
	courses, $e_1$ and $e_2$, and two rooms $r_1$ and $r_2$. The dashed edges
		are optimal assignments. The shown optimal solutions of the two
		\ac{LBOP}s~\subref{fig:ap1} and~\subref{fig:ap2} are not optimal for
		the \ac{GLBOP}~\subref{fig:ap3}.}
\label{fig:rascosts}
\end{figure}
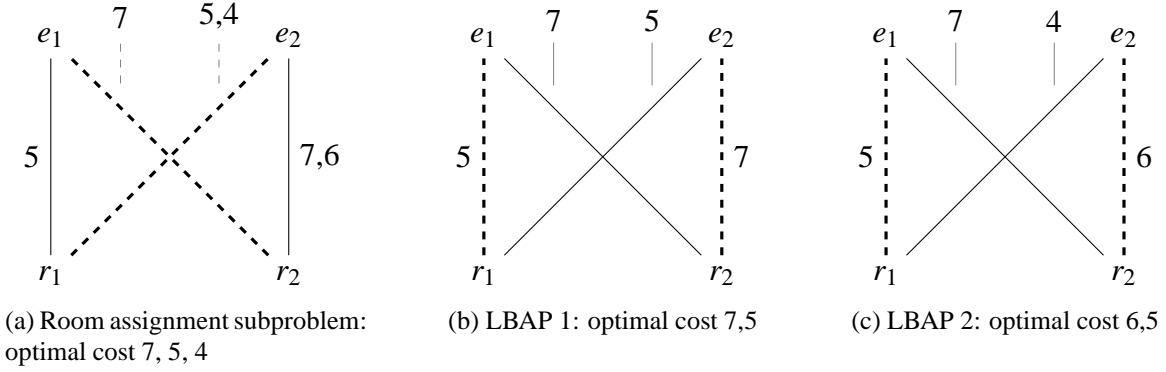

\begin{corollary}
\label{cor:ras}
For a given period $p$, the room assignment subproblem can be solved in time
$O(|U|\cdot T_{\rm LSAP}(n))$, where $n = \max \{ |C_p|, |R| \}$.
\end{corollary}

\begin{proof}
Problem~\eqref{eq:ras} is an assignment problem, just as the room assignment
subproblem of the \ac{CB-CTT} problem. Only the objective function is different.
Hence, combining Theorems~\ref{thm:glbop} and~\ref{thm:ras}
yields the result.
\end{proof}

\begin{remark}
As noted by Lach and L\"ubbecke in~\cite{LL:12}, the room stability constraint
(S4) introduces dependencies between room assignments in different timeslots,
which prevents us from extending the decomposition to more than a single
period.
\end{remark}

\begin{remark} 
In the general case, whenever there is a decomposition of~\eqref{eq:sop} into a master
problem and a subproblem, such that the subproblem can be solved efficiently,
then the subproblem of the corresponding max-min fair optimization problem can
be solved efficiently if it is a \ac{GLBOP}. This observation may be useful
when turning a problem of the form~\eqref{eq:sop} into a \ac{LBOP}.
\end{remark}

In Section~\ref{sec:evaluation}  we will provide experimental evidence that solving the room
assignment subproblem to optimality is useful for improving the performance of
a heuristic algorithm for the \ac{MMF-CB-CTT} problem.


%% file: 04_quantify.tex
\section {Quantifying Max-min Fairness}
\label{sec:iso}

When dealing with randomized optimization algorithms, one can employ a wealth
of statistical tools to extract meaningful information about algorithms'
absolute and relative performance. These tools include statistical tests such
as the Wilcoxon rank-sum test and measures such as the mean quality of the
solutions, the standard deviation, the median quality, the quality of the best
and worst solutions, and so on. Due to the qualitative nature of max-min
fairness, so far only statistical tools based on ranking can be used for
evaluating randomized max-min fair optimization algorithms. In this section we
propose a novel approach to partially overcome this limitation. Similar to
the problem \eqref{eq:glbop} from the previous section, we consider combinatorial
problems such that the cost of a feasible solution is a finite multiset.  The
main idea is to construct an isomorphism from the cost multisets ordered by
$\mmf$, to an interval of the natural numbers ordered by the usual $<$
relation. Using this isomorphism, we can perform all operations on natural
numbers, and retrieve the corresponding cost multiset. This means that if we
have a set of allocation vectors, we can determine for example an allocation
vector close to the average allocation.  Thus, in our experiments in the next
section we will be able to compare the average solution quality of two
algorithms for the \ac{MMF-CB-CTT} problems.

%

Let $k \in \mathbb{N}$ and $N = \{0,\ldots,k\}$. Further, let
$\sortedsequences_n(N)$ denote the non-increasing sequences of length $n$ over
the alphabet $(N, <)$. 

\begin{lemma}
	\label{lemma:rank}
	Let $\rank: \sortedsequences_n(N) \rightarrow \mathbb{N}_0$ be a mapping
	such that for any $s \in \sortedsequences_n(N)$, $s = x_1,\ldots,x_n$,
	\begin{equation}
		\label{eq:rank}
		\rank(s) = \sum_{i=1}^{n} \binom{n+x_i-i}{x_i-1} \enspace.
	\end{equation}
	Then $\rank$ is an isomorphism $(\sortedsequences_n(N), \lex) \rightarrow
	(\mathbb{N}_0, <)$.
\end{lemma}
 
\begin{proof}
$(\sortedsequences_n(N), \lex)$ is a linearly ordered set with least element
$(0,\ldots,0)$. Since $\sortedsequences_n(N)$ is linearly ordered by $\lex$,
there is a unique number $r_s$ for each $s \in \sortedsequences_n(N)$, which is
the cardinality of the set $\{s' \in \sortedsequences_n(N) \mid s' \lex s\}$.
Thus, the function mapping each $s \in \sortedsequences_n(N)$ to $r_s$ is a
bijective mapping and it is order-preserving as required. It remains to be
shown that $\rank(s)$ computes $r_s$ for all $s \in \sortedsequences_n(N)$.

Let $s = (x_1,\ldots,x_n) \in \sortedsequences_n(N)$. The value $r_s$ can be
determined by the following recursion:
\begin{equation}
	\label{eq:rank_rek}
	r_{x_1,\ldots,x_n} = r_{x_2,\ldots,x_n} + r_{x_1,0,\ldots,0} \enspace.
\end{equation}
This recursion separately counts the non-increasing sequences $\{ s' \in
\sortedsequences_n(N) \mid (x_1,0,\ldots,0) \lexeq s' \lex s \}$ and  $\{s' \in
\sortedsequences_n(N) \mid  s' \lex (x_1,0,\ldots,0) \}$.
The number of sorted sequences of length $n$ over an ordered alphabet
of size $k$ is
$\binom{n+k}{k}$. Thus, 
$$
   r_{\underbrace{x_1, 0,\ldots,0}_{\text{length }n}} = \binom{n+x_1-1}{x_1-1} \enspace.
$$
In particular, for $s \in \sortedsequences_1(N)$ we have $r_x = \binom{x}{x-1} = x$. 
Unfolding the recursion~\eqref{eq:rank_rek} yields~\eqref{eq:rank}. Therefore,
		  $\rank(s)$ computes $r_s$ for each $s \in \sortedsequences_n(N)$.
\end{proof}

The argument above can be extended to non-increasing sequences of
length \emph{at most} $n$ by choosing as alphabet $N \cup \{-\infty\}$ and
constructing from each sequence of length less than $n$ a sequence of length
$n$ by padding it with $-\infty$.
Please note that the alphabet can be any finite totally ordered set $(A,
<_A)$, since it is isomorphic to $\{0,\ldots,k\}$ for some $k\in
\mathbb{N}$.  

Let $\mathcal{M}_n(N)$ denote the finite multisets of cardinality $n$ over a
finite alphabet $(N, <)$. Similar to the problem~\eqref{eq:glbop}, consider an
instance of some combinatorial minimization problem with feasible solutions
$\solutions$ and a cost function of the form
\[
	w: \solutions \rightarrow \mathcal{M}_n(N)\enspace.
\]
Let $\qes x$ denote the non-increasing sequence of length $n$ containing the
items of $x \in \mathcal{M}_n(N)$.

\begin{theorem}
	\label{thm:iso}
	The mapping 
	\begin{align*}
	\rho: (\mathcal{M}_n(N), \mmf) & \rightarrow  (\mathbb{N}_0, <) \\ 
	\rho(x) & \mapsto \rank(\qes x)\enspace,
	\end{align*}
	is an isomorphism. 
\end{theorem}
\begin{proof}
Each multiset $x \in \mathcal{M}_n(N)$ can be represented in a unique way as a
non-increasing sequence of length $n$. Thus, we can apply the isomorphism \rank
from Lemma~\ref{lemma:rank} to $\qes x$.
\end{proof}

We establish a similar result for maximization problems. In this context, the
fairness of two allocations $x, y \in \mathcal{M}_n(N)$ can be compared as
follows: $x$ is fairer than $y$, denoted by $x \mmfmax y$, if $\seq y \lex \seq
x$, where $\seq x$ denotes the non-decreasing sequence of the items in $x$. Let
$m = \max\{ N \},\ldots,\max\{ N \} \in
\sortedsequences_n(N)$, and let $a-b$ denote the element-wise substraction of
two sequences $a, b \in \sortedsequences_n(N)$. Furthermore, let $\seq x$
denote the non-decreasing sequence of length $n$ containing the items of $x \in
\mathcal{M}_n(N)$.

\begin{theorem}
	The mapping
	\begin{align*}
	\rho: (\mathcal{M}_n(N), \mmfmax) & \rightarrow  (\mathbb{N}_0, <) \\ 
	\rho(x) & \mapsto \rank(m - \seq x)\enspace,
	\end{align*}
	is an isomorphism. 
\end{theorem}
\qed

%% file: 05_evaluation.tex
\section{Evaluation}
\label{sec:evaluation}

In this section we are going to present experimental evidence for the
usefulness decomposition presented in Section~\ref{sec:decomp}. We compare the
performance of two randomized heuristic algorithms for the \ac{MMF-CB-CTT} problem,
both of which are based on the algorithm \MMFSA from~\cite{MW:12}. The first
algorithm is the
one that performed best in~\cite{MW:12}. It uses a decomposition of the
\ac{MMF-CB-CTT} problem that is similar to the one presented in this work, but
models the room assignment subproblem as \ac{LSAP}. Thus, we will refer to this
algorithm by \MMFSANOGIP. The second algorithm, \MMFSAGIP, uses the \ac{MMF-CB-CTT}
decomposition from Section~\ref{sec:decomp} and thus solves~\eqref{eq:glbop} to obtain
an optimal room assignment for a given period. Apart from how the room
assignment subproblem is solved, the the two algorithms are identical. We
compare both algorithms with respect to the best solutions they produce as well
as the average solution quality per instance. In order to determine average
results, we use the isomorphism $\rho$ from Theorem~\ref{thm:iso} as
described in the previous section. Our results indicate that the algorithm
which solves the \ac{GLBOP} room assignment subproblem significantly
outperforms the other one.

The algorithm \MMFSA is a variant of simulated annealing
(SA)~\cite{Kirkpatrick:83} which has been tailored to the \ac{MMF-CB-CTT} problem.
Simulated annealing iteratively generates new candidate solutions and keeps (or
		\emph{accepts}) a new solution if it is better. If the new solution is
worse, then it is accepted with a certain probability which depends on the
temperature \TEMP. There are three crucial design choices when adapting
simulated annealing to a particular problem: The cooling schedule, the
acceptance criterion, and the neighborhood structure. In both algorithms under
consideration, \MMFSAGIP and \MMFSANOGIP, we use the standard geometric cooling
schedule, which lets the temperature decay exponentially from a given
$\TEMP_{max}$ to a given $\TEMP_{min}$. Both algorithms use the acceptance
criterion based on the component-wise energy difference, which performed best
in a comparison of different acceptance criteria in~\cite{MW:12}. Also, both
algorithms use a neighborhood structure based on the well-known Kempe-move. A
Kempe-move swaps a subset of the events assigned to two given periods such
that the conflict constraints between the events are not violated. The
difference between both algorithms is how the room-assignment subproblems are
solved after performing a Kempe-move: \MMFSANOGIP solves two \ac{LSAP}s in order to
assign rooms efficiently, while \MMFSAGIP solves two \ac{GLBOP}s. Thus, the second
room assignment performed by \MMFSAGIP is optimal assuming that the rest of the
timetable is fixed. Our evaluation shows that this is beneficial for the
overall algorithm performance.

In order to compare the two algorithms, we performed 50 independent runs for
each algorithm on the 21 \ac{CB-CTT} instances from track three of the
timetabling competition ITC2007~\cite{ITC2007:CB-CTT}.
In each run, $10^6$ iterations of the simulated annealing procedure were performed. 
We did not tweak the temperature-related parameters of the algorithms
extensively, but determined experimentally that $\TEMP_{max} = 5$ and
$\TEMP_{min} = 0.01$, as suggested in~\cite{MW:12} work well.
Table~\ref{tab:best_results} shows the best allocation vectors obtained by
\MMFSANOGIP and \MMFSAGIP on the 21 \ac{CB-CTT} instances. To aid the
presentation of the results, Table~\ref{tab:best_results} shows the allocation
vectors in a compressed form: The penalty values are sorted in non-increasing
orders the multiplicities of the values are shown as exponents. For example, a table
entry of $5^210^11$ denotes an allocation vector in which the penalty value 5
appears two times, 1 appears one and 0 eleven times. Note that from the max-min
fairness perspective it is not important to which curriculum a penalty value
corresponds, so we omit this information. The results in
Table~\ref{tab:best_results} show that \MMFSAGIP finds better solutions than
\MMFSANOGIP on 18 instances, while the best solutions found by \MMFSANOGIP are better on
\texttt{comp04} and \texttt{comp18}. In addition to the best allocation vectors,
Table~\ref{tab:best_results} shows the average allocations over the 50 runs for
each instance. The average allocations have been computed using $\rho$ from
Section~\ref{sec:iso}: The allocation vectors were mapped to the natural
numbers, then the average was calculated and rounded to the nearest integer.
Finally, the result was mapped back to the corresponding equivalence class of
allocation vectors. A comparison of the average allocations shows that in this
respect, \MMFSAGIP outperforms \MMFSANOGIP on 16 instances while it is beaten
on the instances \texttt{comp05}, \texttt{comp08}, \texttt{comp15} and
\texttt{comp21}. 

We also performed the one-sided Wilcoxon rank-sum test with a significance
level of 0.01. According to the test, \MMFSAGIP yields significantly better
results than \MMFSANOGIP on instances \texttt{comp02}, \texttt{comp06},
		\texttt{comp07}, \texttt{comp08},
\texttt{comp10}, \texttt{comp13}, \texttt{comp14}, 
\texttt{comp17}, \texttt{comp18}, \texttt{comp19}, \texttt{comp20} and
\texttt{comp21}. In contrast, \MMFSANOGIP is not significantly better than
\MMFSAGIP on any of the 21 instances with this significance level. The results
of the Wilcoxon test are consistent with the data in
Table~\ref{tab:best_results}.

\begin{sidewaystable}[p] 
\centering
\footnotesize

\caption{\label{tab:best_results}Comparison of the best and average objective values of the solutions found by \MMFSANOGIP and \MMFSAGIP on
the 21 \ac{CB-CTT} instances from~\cite{ITC2007} for 50 independent runs per instance and per algorithm. For each instance, the best results and best average results are marked in bold face.}

\bigskip

\begin{tabular}{@{}r|c|c|c|c@{}}
& \multicolumn{2}{|c|}{\MMFSAGIP} & \multicolumn{2}{|c}{\MMFSANOGIP} \tabularnewline\hline
Instance			& best & average & best & average \tabularnewline\hline
\tt{comp01.ectt}	& $\bf 5^{2},0^{12}$	& $\bf 6,5,4^{5},3^{4},2,1,0$	& $5^{2},1,0^{11}$	& $6,5^{3},4,2^{2},1^{5},0^{2}$\tabularnewline
\tt{comp02.ectt}	& $\bf 2^{32},1^{5},0^{33}$	& $\bf 4^{5},3^{18},2^{16},1^{18},0^{13}$	& $4,3^{3},2^{26},1^{12},0^{28}$	& $5^{2},4^{30},3^{4},2^{17},1^{16},0$\tabularnewline
\tt{comp03.ectt}	& $\bf 6^{4},4^{11},2^{23},1,0^{29}$	& $\bf 6^{5},5^{40},4^{8},3^{2},2^{4},0^{9}$	& $6^{4},4^{12},2^{27},1^{3},0^{22}$	& $6^{10},5^{9},4,3^{24},2,1^{3},0^{20}$\tabularnewline
\tt{comp04.ectt}	& $6^{4},4^{2},2^{4},1^{7},0^{40}$	& $\bf 6^{4},4^{3},3^{6},2^{34},1^{8},0^{2}$	& $\bf 6^{4},4^{2},2^{4},0^{47}$	& $6^{4},4^{3},3^{11},2^{33},1^{4},0^{2}$\tabularnewline
\tt{comp05.ectt}	& $\bf 19^{2},18^{3},17^{3},16^{5},15^{2},14^{8},\ldots$	& $20^{3},18^{6},16,15^{9},14^{2},13^{3},\ldots$	& $19^{2},18^{3},17^{3},16^{6},15^{2},14^{7},\ldots$	& $\bf 20^{2},19^{4},18^{11},16^{3},15^{30},14^{7},\ldots$\tabularnewline
\tt{comp06.ectt}	& $\bf 12,4,2^{29},1^{15},0^{24}$	& $\bf 12,5^{6},4^{2},3^{9},2^{43},1^{4},0^{5}$	& $12,4^{3},3,2^{31},1^{12},0^{22}$	& $12,5^{11},4^{3},3^{5},2^{4},1^{5},0^{41}$\tabularnewline
\tt{comp07.ectt}	& $\bf 6,2^{14},1^{18},0^{44}$	& $\bf 6,4,3^{21},2^{9},1^{15},0^{30}$	& $6,3,2^{30},1^{25},0^{20}$	& $6,4^{2},3^{22},2^{7},1^{11},0^{34}$\tabularnewline
\tt{comp08.ectt}	& $\bf 6^{4},4^{2},2^{9},1^{6},0^{40}$	& $6^{4},4^{3},3^{23},2^{15},1^{14},0^{2}$	& $6^{4},4^{2},2^{11},1^{4},0^{40}$	& $ 6^{4},4^{4},3^{3},2^{9},1^{20},0^{21}$\tabularnewline
\tt{comp09.ectt}	& $\bf 6^{9},4^{11},2^{18},1^{2},0^{35}$	& $\bf 6^{10},5^{3},4^{8},3^{21},2^{12},1^{13},0^{8}$	& $6^{9},4^{14},2^{12},1^{6},0^{34}$	& $6^{10},5^{5},4^{4},3^{10},2^{10},1^{17},0^{19}$\tabularnewline
\tt{comp10.ectt}	& $\bf 2^{12},1^{5},0^{50}$	& $\bf 4,3^{34},2^{13},0^{19}$	& $2^{20},1^{9},0^{38}$	& $4^{10},3^{4},2^{24},1^{6},0^{23}$\tabularnewline
\tt{comp11.ectt}	& $0^{13}$	& $0^{13}$	& $0^{13}$	& $0^{13}$\tabularnewline
\tt{comp12.ectt}	& $\bf 10^{4},8^{26},7,6^{50},5^{2},4^{38},\ldots$	& $\bf 12^{2},10^{2},9^{45},8^{2},7^{12},6^{11},\ldots$	& $10^{5},8^{26},6^{48},5^{6},4^{42},3^{4},\ldots$	& $12^{3},11^{19},10^{9},9^{2},8^{21},7^{12},\ldots$\tabularnewline
\tt{comp13.ectt}	& $\bf 6^{6},4^{4},2^{12},1^{2},0^{42}$	& $\bf 6^{6},4^{6},3^{11},2^{24},1^{6},0^{13}$	& $6^{6},4^{4},2^{21},1^{4},0^{31}$	& $6^{6},4^{7},3^{11},2,1^{29},0^{12}$\tabularnewline
\tt{comp14.ectt}	& $\bf 8^{4},4^{3},2^{13},1^{3},0^{37}$	& $\bf 8^{4},4^{5},3^{5},2^{32},1^{11},0^{3}$	& $8^{4},4^{3},2^{16},1^{2},0^{35}$	& $8^{4},4^{13},3^{22},2^{10},1^{7},0^{4}$\tabularnewline
\tt{comp15.ectt}	& $\bf 6^{4},4^{10},2^{22},1^{4},0^{28}$	& $6^{11},5^{3},4^{8},3^{13},2^{7},1^{2},0^{24}$	& $6^{4},4^{11},3,2^{25},1^{5},0^{22}$	& $\bf 6^{8},5^{3},4^{18},3,2^{12},1^{18},0^{8}$\tabularnewline
\tt{comp16.ectt}	& $\bf 4^{5},2^{14},1^{3},0^{49}$	& $\bf 5^{5},4^{10},3^{6},2^{19},1^{6},0^{25}$	& $4^{5},3,2^{17},1^{7},0^{41}$	& $5^{12},4^{4},3^{16},2^{4},1^{22},0^{13}$\tabularnewline
\tt{comp17.ectt}	& $\bf 10^{2},6^{2},4^{5},2^{31},1^{8},0^{22}$	& $\bf 10^{2},6^{2},5,4^{32},2^{6},0^{27}$	& $10^{2},6^{2},4^{8},3,2^{25},1^{12},0^{20}$	& $10^{2},6^{2},5^{9},4^{29},3^{16},2^{7},1^{5}$\tabularnewline
\tt{comp18.ectt}	& $6,4^{18},3,2^{15},1^{4},0^{13}$	& $\bf 7,6^{19},5^{6},4^{3},3^{2},2^{9},1^{11},0$	& $\bf 6,4^{17},3^{3},2^{14},1^{3},0^{14}$	& $7^{4},5^{5},4^{20},3^{9},2^{4},1^{9},0$\tabularnewline
\tt{comp19.ectt}	& $\bf 6^{4},4^{6},2^{12},1^{5},0^{39}$	& $\bf 6^{14},5^{9},4^{7},2^{18},1^{6},0^{12}$	& $6^{4},4^{7},2^{14},1^{10},0^{31}$	& $7^{2},6^{4},5^{16},4^{2},3^{19},2^{4},1^{16},0^{3}$\tabularnewline
\tt{comp20.ectt}	& $\bf 4^{3},3^{2},2^{27},1^{16},0^{30}$	& $\bf 5^{2},4^{24},3^{4},2,1,0^{46}$	& $4^{5},3^{2},2^{34},1^{13},0^{24}$	& $5^{16},4^{6},3^{2},1^{7},0^{47}$\tabularnewline
\tt{comp21.ectt}	& $\bf 10,6^{4},4^{15},3,2^{25},1^{7},0^{25}$	& $10,6^{10},5^{33},4^{4},2,1^{6},0^{23}$	& $10,6^{4},5,4^{16},2^{27},1^{4},0^{25}$	& $\bf 10,6^{9},5^{15},4^{6},3,2^{31},0^{15}$\tabularnewline

\end{tabular}
\end{sidewaystable}

In contrast to experimental setup in~\cite{MW:12}, we did not use a timeout,
but set a fixed number of iterations for the direct comparison of \MMFSAGIP and
\MMFSANOGIP. The reason for this decision is that in practice, solving~\eqref{eq:glbop}
takes significantly more time than solving a \ac{LSAP}. The increase in
runtime is  due to overhead required by construction of the cost vectors
($O(|U|^2)$) and the $O(|U|)$ factor for solving the linear vector assignment
problem (see Corollary~\ref{cor:ras}).  Since we are mainly interested in the
implications of modelling the room assignment subproblem as a \ac{GLBOP} instead
of an \ac{LSAP}, both algorithms should be able to solve a similar number of
room-assignment subproblems. From the data shown in
Table~\ref{tab:best_results} we can conclude, that using the decomposition
presented in Section~\ref{sec:decomp} is clearly the smarter choice, since,
after solving equally many subproblems, it produces superior results
compared to the approach from~\cite{MW:12}. However, if we employ the
timeout as it was required for the ITC2007 competition
(see~\cite{ITC2007:CB-CTT}), \MMFSANOGIP is the better choice,
because it can perform significantly more iterations within the given
timeout. 


%% file: 06_conclusion.tex
\section{Conclusion}
\label{sec:conclusion}

In this work we proposed a decomposition of the \ac{MMF-CB-CTT} problem
from~\cite{MW:12}. The decomposition models the room assignment subproblem as a
assignment problem with a generalized lexicographic bottleneck objective. Using
this decomposition, the room assignment subproblem can be solved in polynomial
time. We use this result to improve the performance of the \MMFSA algorithm
proposed in~\cite{MW:12}, which originally modelled the room assignment
as an \ac{LSAP}. 
In our experiments we compare the performance of two variants of the
\MMFSA algorithm, wich differ with respect to how the room assignment
subproblem is performed. Our results indicate that using the decomposition
proposed in Section~\ref{sec:decomp} improves the performance of the \MMFSA
algorithm on most of the ITC2007 benchmark instances.  Furthermore, we proposed
a measure for quantifying how fair a timetable is with respect to max-min
fairness. Using this measure helps to apply statistical methods in the analysis
of the performance of randomized optimization algorithms for optimization
problems with a bottleneck objective. In particular, it enables us to compare
the average solution quality of the two variants of the \MMFSA algorithm. The
results indicate that using the new decomposition, \MMFSA produces better
timetables on average for 16 out of 21 instances.
